\documentclass[nohyperref]{article}

\usepackage{comment}
\usepackage{microtype}
\usepackage{graphicx}
\usepackage{subfigure}
\usepackage{booktabs} 
\usepackage{bbm}
\usepackage{physics}

\usepackage{hyperref}

\usepackage[accepted]{icml2022_TAGML}

\usepackage{amsmath}
\usepackage{amssymb}
\usepackage{mathtools}
\usepackage{amsthm}

\usepackage[capitalize,noabbrev]{cleveref}

\theoremstyle{plain}
\newtheorem{theorem}{Theorem}[section]
\newtheorem{proposition}[theorem]{Proposition}
\newtheorem{lemma}[theorem]{Lemma}
\newtheorem{corollary}[theorem]{Corollary}
\theoremstyle{definition}

\theoremstyle{remark}
\newtheorem{remark}[theorem]{Remark}

\usepackage[textsize=tiny]{todonotes}

\newcommand{\nm}[1]{{\|#1\|}}
\newcommand{\R}{\mathbb R}

\newcommand{\invf}{F^{\dagger}}

\date{}
\begin{document}

\twocolumn[
\icmltitle{Geodesic Properties of a Generalized Wasserstein Embedding for Time Series Analysis}

\begin{icmlauthorlist}
\icmlauthor{Shiying  Li}{}
\icmlauthor{Abu Hasnat Mohammad Rubaiyat}{}
\icmlauthor{Gustavo K. Rohde}{}
\end{icmlauthorlist}


\icmlcorrespondingauthor{Shiying Li}{sl8jx@virginia.edu}
\icmlcorrespondingauthor{Abu Hasnat Mohammad Rubaiyat}{ar3fx@virginia.edu}


\vskip 0.3in
]






\begin{abstract}
Transport-based metrics and related embeddings (transforms) have recently been used to model signal classes where nonlinear structures or variations are present. In this paper, we study the geodesic properties of time series data with a generalized Wasserstein metric and the geometry related to their signed cumulative distribution transforms in the embedding space. Moreover, we show how understanding such geometric characteristics can provide added interpretability to certain time series classifiers, and be an inspiration for more robust classifiers.   
\end{abstract}
\section{Introduction}
\label{intro}
%

Transport-based distances, such as Wasserstein distances \cite{villani2003topics},  have been shown to be an effective tool in signal analysis and machine learning applications including image retrieval \cite{rubner2000earth} and  registration \cite{haker2004optimal}, modeling biological morphology \cite{ozolek2014accurate, basu2014detecting},  comparing probability distributions \cite{arjovsky2017wasserstein}, and  providing good low-dimensional embeddings of image manifolds \cite{hamm2022wassmap}, to name a few. The success of transport-based distances are in part due to their ability to incorporate information of spatial or time deformations naturally structured in signals \cite{kolouri2017optimal}. For example, they are observed to  correctly recover biologically interpretable and statistically significant differences \cite{basu2014detecting} and  quantify semantic differences between distributions that correlate well with human perception \cite{rubner2000earth}.



Within the set of problems in data science, modeling of time series data is considered a challenging problem. Deformation-based methods such as dynamic time warping \cite{abanda2019review, lines2015time} have been shown successful in enabling the comparison of time series data more meaningfully. In addition to being able to align features from two different time series, Wasserstein-type distances are also true distances that can allow for a low-dimensional representation of dynamical systems in which time series can be classified and statistically analyzed \cite{muskulus2011wasserstein}. 
In recent years, many transport transform-based techniques have been developed to leverage Wasserstein distances and linearized ($L^2$) embeddings to facilitate the application of many standard data analysis \cite{kolouri2017optimal}. In particular, the cumulative distribution transform (CDT), based on the 1D Wasserstein embedding, was introduced in \cite{Park:18} as a means of classifying  normalized non-negative signals, and has been extended to general signed signals via the the signed cumulative distribution transform (SCDT) in \cite{aldroubi2022signed}.


 Wasserstein embeddings based on the cumulative distribution transform (CDT) \cite{Park:18,Rubaiyat:20,aldroubi2022signed} have recently emerged as a robust, computationally efficient, and accurate end-to-end classification method for time series (1D signal) classification. They are particularly effective for classifying data emanating from physical processes where signal classes can be modeled as observations of a particular set of template signals  under some unknown, possibly random, temporal deformation or transportation \cite{Park:18,shifat2020radon,rubaiyat2022nearest}. Efforts have been made to explain the success of these models by understanding the geometry of the transform embedding space \cite{Park:18, aldroubi2021partitioning,moosmuller2020linear}, where embedding properties and conditions  when the data class becomes convex and linearly separable in the transform space are studied. In a nutshell, the template-deformation-based generative models capture the nonlinear structure of signals and the nonlinear transport transforms render signal classes that are nonlinear and non-convex into convex sets in transform embedding space (see Figure \ref{fig:scdt_embedding}). 

As the CDT \cite{Park:18} is defined for probability measures (or their associated density functions) which are non-negative and normalized, in this manuscript we elucidate the geometry of general time series (signed, non-normalized) with a generalized Wasserstein metric in the  SCDT embedding space with a $L^2$-type metric. In addition to the convexity and isometric embedding properties of the SCDT \cite{aldroubi2022signed}, which are shared by the CDT, we look at the geodesic properties of the SCDT and illustrate the differences between geometry of data in CDT and SCDT space. In particular, the SCDT embedding $\widehat S\subseteq \big(L^2(s_0)\times\R\big)^2$ is not a geodesic space while geodesics exist between signals within a generative model (see Figure \ref{fig:scdt_embedding}). As a preliminary application, we  provide an interpretation of the nearest (transform) subspace classifiers proposed in \cite{rubaiyat2022nls,rubaiyat2022nearest} through  visualizing  paths between the test signals and their projections to various subspaces. In particular, we illustrate that using the training samples, these classifiers ``correctly" generate a (local) subspace that models the generative clusters (see \eqref{eq:genmodel}) to which the given test signals belong. We hypothesise this knowledge can lead to the design of more efficient and accurate pattern recognition tools and added interpretability of various classifiers. 

\begin{figure}[h!]
    \centering
    \includegraphics[width =7cm,height = 2.8cm]{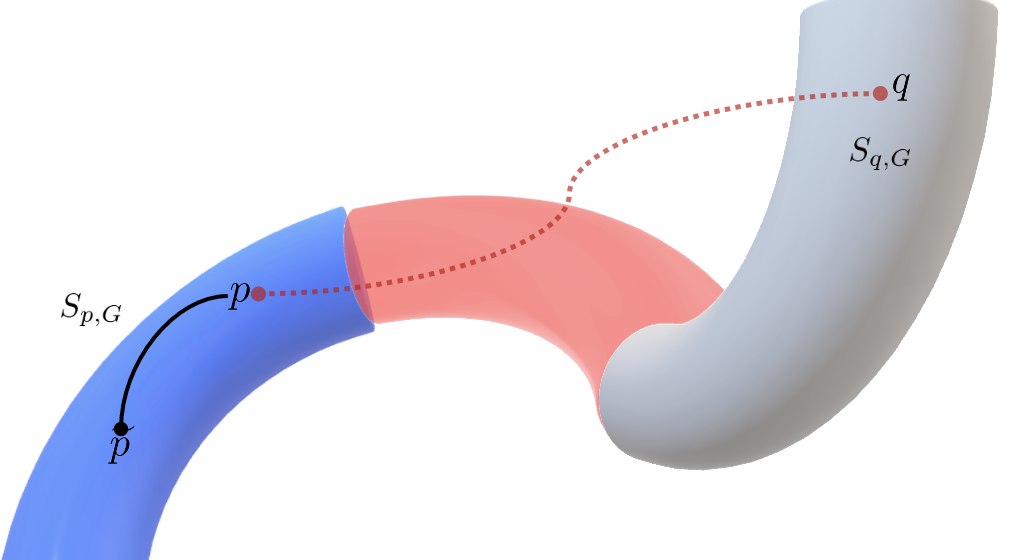}
    \vspace{.5em}
    
    \includegraphics[width=7cm, height = 4.5cm]{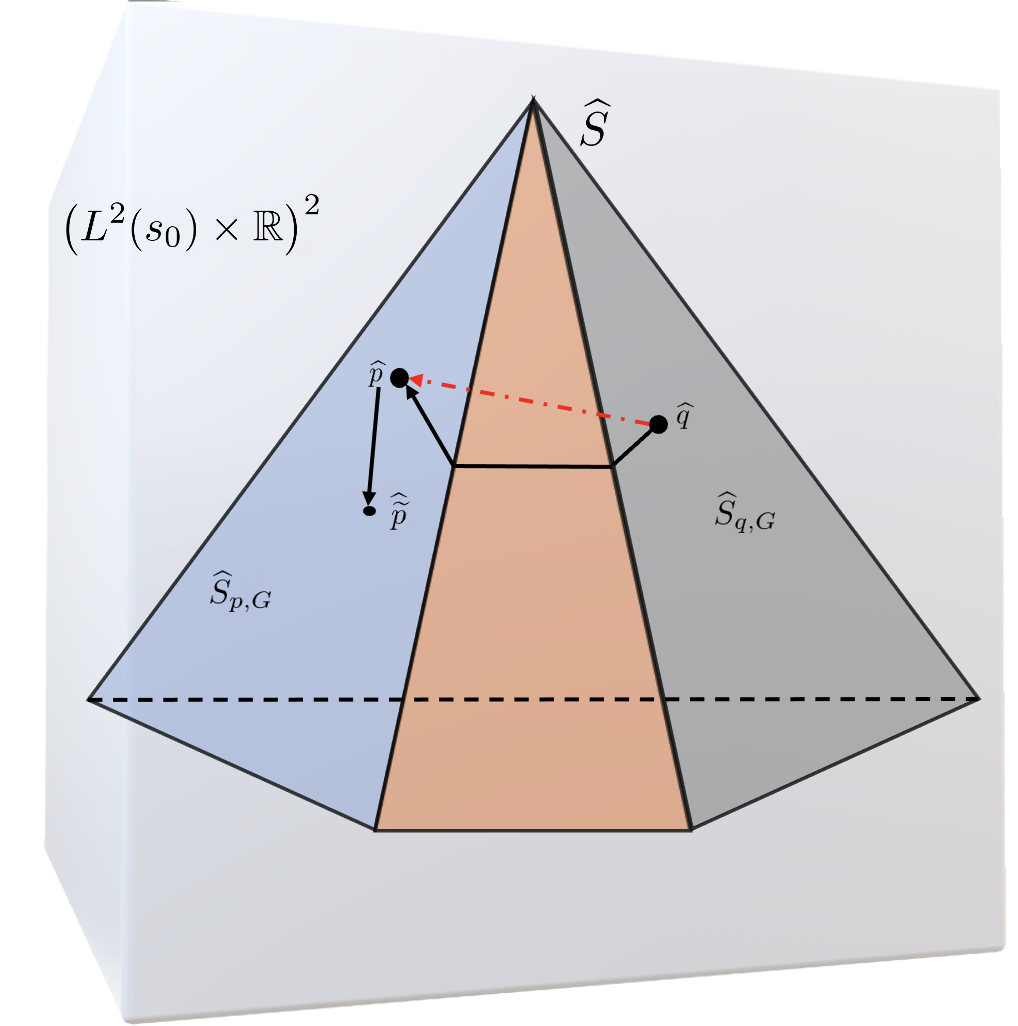}

    \caption{Space of time series with a generalized Wasserstein metric (top) and its embedding $\widehat S$ (bottom) are not geodesic spaces while geodesics exist between any pair of signals in the same generative cluster $S_{p,G}$. Here $\widehat S$ is represented as the faces of the polyhedron (see Remark \ref{rmk:notconvex} and \ref{rmk:unionconvex}) and its ambient space $\big(L^2(s_0)\times\R\big)^2$ is represented as the open cube.}
    \vspace{-1.5em}
    \label{fig:scdt_embedding}
\end{figure}




Throughout the manuscript, we work with $L^1$ signals $s$ with finite second moments, \footnote{For bounded $\Omega_s$, it suffices to require that $s\in L^2(\Omega_s)$, in which case by Cauchy-Schwartz $\int_{\Omega_s}t^2|s(t)|dt<\infty$; $s\in L^1(\Omega_s)$ follows from the fact that  $L^2(\Omega_s)\subseteq L^1(\Omega_s)$.} where $\Omega_s\subseteq\mathbb{R}$ is the bounded domain over which $s$ is defined. We denote $S$ as the set of $L^1$ signals with finite second moments and $S_1$  as the set of non-negative $L^1$-normalized signals in $S$  and  $\|s\|$ as the $L^1$-norm of $s$.

\section{Transport Transforms and Geometric Properties of Time Series}
We give a brief overview of the CDT and SCDT and their associated embedding properties in Section \ref{sec:prelims} and present the main theorems about the geodesic properties of time series with respect to  a generalized Wasserstein metric and the geometry of the SCDT embedding  in Section \ref{sec:geodesicthms}. 

\subsection{The Signed Cumulative Distribution Transform and a Generalized Wasserstein Metric}\label{sec:prelims}
The Cumulative Distribution Transform (CDT) was introduced in \cite{Park:18} for non-negative $L^1$-normalized functions. In particular, given non-negative signals $s$ and $s_0$ (a fixed reference) with $\nm{s}=\nm{s_0}=1$, the CDT $s^*$ of $s$ is defined as the optimal transport map  
\footnote{The fact that $s_0$ and $s$ have finite second moments guarantees the existence of a unique optimal transport map between them. } between reference $s_0$ and $s$, which is the unique non-decreasing map \cite{santambrogio2015optimal}
\begin{equation}\label{eq:cdt_alt}
	s^* = \invf_s\circ F_{s_0},
\end{equation} 
where $F_{s_0}$ is the cumulative distribution function of $s_0$ and $\invf_s(x): = \inf\{t\in\R : F(t)\geq x\}$ is  the generalized inverse of $F_s$.\footnote{If $F_s$ is strictly increasing on $\Omega_s$ then $\invf_s = F^{-1}_s$.} In particular, $(s^*)_{\sharp}s_0 = s$ and $s^*$ minimizes the Wasserstein-2 cost  $d_{W^2}(s_0,s):= \sqrt{\inf\limits_{s= T_{\sharp}s_0}\int_{\R} |x-T(x)|^2s_0(x)dx}$ between $s_0$ and $s$. It is well-known \cite{santambrogio2015optimal} that  $\big(S_1, d_{W^2}(\cdot,\cdot)\big)$ is a geodesic space with the constant-speed geodesic between $s$ and $ \tilde s$  given by for  $\alpha \mapsto p_\alpha = \Big((1-\alpha)id+ \alpha T^*\Big)_{\sharp}s$, where $T^*$ is the optimal transport map between $s$ and $\tilde s$ in $S_1$ and $\alpha\in [0,1]$.
Moreover, the CDT defines an isometric embedding from the space of non-negative normalized signals $\big(S_1, d_{W^2}\big)$ to the transform space $\big(S_1^{*}, L^2(s_0)\big)$ (cf. \cite{villani2003topics}), i.e., 
\begin{equation}\label{eq:cdtembedding}
	d_{W^2}(s_1,s_2) = \nm{s_1^*-s_2^*}_{L^2(s_0)}
\end{equation}
where $S_1^{*}$ is the set of CDTs of signals in  $S_1$.

Now for a (non-zero) signed signal, the Jordan decomposition \cite{royden1988real} is applied to $s(t) = s^+(t) - s^-(t)$,\footnote{Here $s^+(t)$ and $s^-(t)$ can be seen as the density functions of corresponding measures in the Jordan decomposition of the measure $\mu_s$ associated with $s$ where $d\mu_s(x) := s(x)dx$ .} where $s^+(t)$ and $s^-(t)$ are the absolute values of the positive and negative parts of the signal $s(t)$. Given a fixed $L^1$-normalized positive reference signal $s_0$ defined on $\Omega_{s_0}$, the  signed cumulative distribution transform (SCDT) \cite{aldroubi2022signed} of $s(t) $  is then defined as the following tuple \footnote{When $s=0$, the CDT $s^*$ is defined to be $0$ and the SCDT  $\widehat s := (0,0,0,0)$.} :
\begin{equation}
      \widehat{s} := \Big((\frac{s^+}{\|s^+\|})^*,\nm{s^+},(\frac{s^-}{\|s^-\|})^*,\nm{s^-}\Big),
    \label{eq:scdt}
\end{equation}
where $(\frac{s^{\pm}}{\|s^{\pm}\|})^*$ is the CDT (defined in eqn. (\ref{eq:cdt_alt})) of the normalized signal $\frac{s^{\pm}}{\|s^{\pm}\|}$. To simplify notations, from now on, for non-negative signals (e.g., $s^{+}$), its CDT is defined as the CDT of its normalized version, i.e., we denote $(\frac{s^+}{\|s^+\|})^*$ simply as $(s^{+})^*$. Equivalently, \eqref{eq:scdt} becomes $ \widehat{s} = \big((s^+)^*,\nm{s^+},(s^-)^*,\nm{s^-}\big)$. 
 Like the CDT, the SCDT is invertible with the inverse transform (see \ref{append:invSCDT} for more details) given by \footnote{Here $(s^{\pm})^*_{\sharp}s_0\ $ denote the push-forward signal of $s_0$ by the maps $(s^{\pm})^*$ respectively. In particular, when the maps $(s^{\pm})^*$ are differentiable, the inverse formula becomes 
 \vspace{-.2cm}
  \begin{align}
    s(t) = &\|s^{+}\|\Big((s^{+})^{*^{-1}}\Big)^{\prime}(t)s_0\Big( (s^{+})^{*^{-1}}(t)\Big) \nonumber\\ 
    &- \|s^{-}\|\Big((s^{-})^{*^{-1}}\Big)^{\prime}(t)s_0\Big( (s^{-})^{*^{-1}}(t)\Big).
\end{align} }
   \begin{align}
 	s = \nm{s^{+}}\Big((s^+)^*_{\sharp}s_0\Big) - \nm{s^{-}}\Big((s^-)^* _{\sharp}s_0\Big).
 \end{align}

Similarly, a generalized Wasserstein-2 distance between two time series can be defined as follows:
\begin{align}
	&D_S(s_1,s_2):= \Bigg(d^2_{W^2}\Big(\frac{s_1^{+}}{\nm{s_1^{+}}},\frac{s_2^{+}}{\nm{s_2^{+}}}\Big)+d^2_{W^2}\Big(\frac{s_1^{-}}{\nm{s_1^{-}}},\frac{s_2^{-}}{\nm{s_2^{-}}}\Big)\nonumber\\
	&+\Big(\nm{s_1^{+}}-\nm{{s}_2^{+}}\Big)^2+\Big(\nm{s_1^{-}}-\nm{s_2^{-}}\Big)^2\Bigg)^{\frac{1}{2}}	\label{eq:genWass}\\
	&= \Bigg(\nm{(s_1^{+})^*-(s_2^{+})^*}^2_{L^2(s_0)}+\nm{(s_1^{-})^*-(s_2^{-})^*}^2_{L^2(s_0)}\nonumber\\
	&+\Big(\nm{s_1^{+}}-\nm{{s}_2^+}\Big)^2+\Big(\nm{s_1^{-}}-\nm{s_2^-}\Big)^2\Bigg)^{\frac{1}{2}}\nonumber\\
	&= \nm{\widehat s_1 - \widehat s_2}_{L^2(s_0)\times 
	\R\times L^2(s_0)\times 
	\R},\label{eq:scdtembedding}
\end{align}
where the second-to-last equality follows from the embedding property \eqref{eq:cdtembedding} of the CDT. Hence the SCDT also defines an isometry (embedding) from the space $\big(S,D_S(\cdot,\cdot)\big)$ of time series  to its transform space $\Big(\widehat S, \nm{\cdot}_{(L^2(s_0)\times \R)^2}\Big)$. 

\subsection{Geodesic Properties}\label{sec:geodesicthms}
We first present a sufficient condition on a pair of signals under which a geodesic exists between them.
\begin{theorem}\label{thm:suff_geodesic}
Fix a reference signal $s_0\in S_1$. Given a signal $s\in S$ and a strictly increasing differentiable function  
 $g: \R\rightarrow\R$, there is a constant speed geodesic between $s$ and $\widetilde s := g^{\prime}s\circ g$. In particular, the geodesic $\alpha\mapsto p_{\alpha}$ is given by $\forall \alpha\in [0,1]$, 
 \begin{align}\label{eq:sz}
  p_{\alpha}& = \big[(1-\alpha)\nm{s^{+}}+\alpha\nm{{\widetilde s}^{+}}\big] \Big((1-\alpha)(s^{+})^* + \alpha({\widetilde s}^{+})^*\Big)_{\sharp}s_0 \nonumber\\
 	-& \big[(1-\alpha)\nm{s^{-}}+\alpha\nm{{\widetilde s}^{-}}\big] \Big((1-\alpha)(s^{-})^* + \alpha({\widetilde s}^{-})^*\Big)_{\sharp}s_0.
 \end{align}
 Note that $p_0= s$ and $p_1= \widetilde s$.
 \end{theorem}
 \begin{proof}
 	It suffices to show that given $\alpha_1,\alpha_2\in [0,1]$, $D_S(p_{\alpha_1},p_{\alpha_2}) = |\alpha_1-\alpha_2| D_S(s,\widetilde s)$.
 	By the composition property of the CDT (see \ref{appendix:composition}), $({\widetilde s}^{\pm})^* = g^{-1}\circ (s^{\pm})^*$. Hence we have that
 	\begin{equation*}
 		(1-\alpha)(s^{\pm})^* + \alpha({\widetilde s}^{\pm})^* = \Big((1-\alpha)id + \alpha g^{-1}\Big)\circ (s^{\pm})^*. 
 	\end{equation*}
 Since $(1-\alpha)id + \alpha g^{-1}$ is strictly increasing and $(s^{+})^*_{\sharp}s_0\perp (s^{-})^*_{\sharp}s_0$ (see \ref{append:invSCDT}), it follows that $\Big((1-\alpha)(s^{+})^* + \alpha({\widetilde s}^{+})^*\Big)_{\sharp}s_0\perp \Big((1-\alpha)(s^{-})^* + \alpha({\widetilde s}^{-})^*\Big)_{\sharp}s_0$ (cf. Lemma 5.4 in \cite{aldroubi2022signed}). By  the inverse formula in Proposition \ref{prop:invSCDT},
 		 it is not hard to see  that the expression in \eqref{eq:sz} is the Jordan decomposition of $p_{\alpha}$, i.e., $p_{\alpha}^{\pm}= \big[(1-\alpha)\nm{s^{\pm}}+\alpha\nm{{\widetilde s}^{\pm}}\big] \Big((1-\alpha)(s^{\pm})^* + \alpha({\widetilde s}^{\pm})^*\Big)_{\sharp}s_0$. Hence
 	\begin{align}
 		&D_S^2(p_{\alpha_1},p_{\alpha_2}) =D_S^2(p^{+}_{\alpha_1},p^{+}_{\alpha_2})+D_S^2(p^{-}_{\alpha_1},p^{-}_{\alpha_2})\nonumber\\
 		=&\nm{(p^{+}_{\alpha_1})^*-(p^{+}_{\alpha_2})^*}_{L^2(s_0)}^2+ \nm{(p^{-}_{\alpha_1})^*-(p^{-}_{\alpha_2})^*}_{L^2(s_0)}^2 \nonumber\\ 
 		+&\big( \nm{p_{\alpha_1}^{+}}-\nm{p_{\alpha_2}^{+}}\big)^2 + \big( \nm{p_{\alpha_1}^{-}}-\nm{p_{\alpha_2}^{-}}\big)^2\nonumber\\
 		=&|\alpha_1-\alpha_2|^2\Big(\nm{(s^{+})^*-({\widetilde s}^{+})^*}_{L^2(s_0)}^2+\nm{(s^{-})^*-({\widetilde s}^{-})^*}_{L^2(s_0)}^2\nonumber\\
 		+& \big(\nm{s^{+}}-\nm{{\widetilde s}^+}\big)^2+\big(\nm{s^{-}}-\nm{{\widetilde s}^-}\big)^2\Big)\nonumber\\
 		=& |\alpha_1-\alpha_2|^2D_S^2(s,\widetilde s),\nonumber
 			\end{align}
 			 where the second-to-last equality follows from the fact  that $(p^{\pm}_{\alpha_i})^* = (1-\alpha_i)(s^{\pm})^* + \alpha_i({\widetilde s}^{\pm})^*, i=1,2$ (see  \ref{append:invSCDT} for more details).
 \end{proof}
 \begin{remark}
 	The same conclusion holds if  $\widetilde s := \lambda g^{\prime}s\circ g$ for  any constant $\lambda >0$ and $g$  strictly increasing and differentiable.
 \end{remark}

\begin{corollary}\label{cor:gengeo}
	Let $G=\{g:\R\rightarrow\R: g ~\textrm{is strictly increasing and differentiable}\}$ and $p\in S$. Consider the following template-deformation based generative model  \cite{Park:18}
\begin{equation}\label{eq:genmodel}
    S_{p,G}:=\{p_g= g^{\prime}p\circ g: g\in G \}.
\end{equation}
	Then for any two signals in  $S_{p,G}$, a unique constant-speed geodesic exists with respect to the generalized Wasserstein-2 metric defined in \eqref{eq:genWass}  . 
\end{corollary}
\begin{proof}
	Let $p_i = g_i^{\prime}p\circ g_i\in S_{p,G}, i= 1,2$. It follows that $p_2 = (g_1^{-1}\circ g_2)^{\prime}p_1\circ (g_1^{-1}\circ g_2)$. Since $g_1^{-1}\circ g_2$ is also strictly increasing, by Theorem \ref{thm:suff_geodesic}, there exists a constant speed geodesic between $p_1$ and $p_2$.
\end{proof}

\begin{remark}\label{rmk:convexity}
More generally, one can show the existence of geodesics between pairs of signals in any subset $U$ of $S$ such that $\widehat U$ is convex, the proof of which is presented in \ref{prop:geoconvex}. 
By the convexity property of the SCDT \cite{aldroubi2022signed}, $\widehat S_{p,G}$ is convex if and only if $G^{-1}$ is convex. Hence 
	Corollary \ref{cor:gengeo} can be seen as a special case of Proposition \ref{prop:geoconvex} since the set $G=G^{-1}$ is convex for $G=\{g:\R\rightarrow\R: g ~\textrm{is strictly increasing and differentiable}\}$. 
\end{remark}

Next we show that a geodesic may not exist between arbitrary two signals in $S$ with the generalized Wasserstein metric. 
 \begin{theorem}
 	The metric space $\big(S,D_S(\cdot,\cdot)\big)$ is not a geodesic space.
 \end{theorem}
 \begin{proof}
 Recall that a metric space is geodesic means that for any two given elements, there exists a path between  such that the length of the path equals the distance between them. We prove this theorem by giving an example where the length of any path between some $s_1,s_2\in S$ is larger than their distance $D_S(s_1,s_2)$.   Let $s_1 = \mathbbm{1}_{[-1,0)}-\mathbbm{1}_{[0,1]}$ and $s_2 = \mathbbm{1}_{[0,1]}-\mathbbm{1}_{[-1,0)}$. \footnote{Here $\mathbbm{1}_{[a,b]}$ denotes the indicator function of interval $[a,b]$.} Assume by contradiction that there is a path $\gamma: [0,1]\rightarrow S$ where $\gamma_0 = s_1$ and $\gamma_1 = s_2$ such that $\textrm{Len}(\gamma)= D_S(s_1,s_2)$. By the embedding property \eqref{eq:scdtembedding} of the SCDT,  it follows that  $\widehat\gamma_{\alpha}$($ \alpha\in [0,1]$) defines a geodesic in $\widehat S\subseteq (L^2(s_0)\times \R)^2 $. Since the space $(L^2(s_0)\times \R)^2$ is uniquely geodesic (see \ref{cor:unigeo}), then $\widehat \gamma_{\alpha} = (1-\alpha)\widehat s_1+\alpha \widehat s_2$. Let $\widehat s_1 = (f_1^{+},1,f_1^{-},1)\in \widehat S$ and it is not hard to see that $\widehat s_2 = (f_1^{-},1,f_1^{+},1)$. Hence $\widehat \gamma_{1/2}:= (f_{1/2}^{+}, a,f_{1/2}^{-},b )= (\frac{f_1^{+}+f_1^{-}}{2},1,\frac{f_1^{+}+f_1^{-}}{2},1)$, which is a contradiction  to the fact that $\widehat \gamma_{1/2}\in \widehat S$ (since  by definition $(f_{1/2}^{+})_{\sharp}s_0$ and $(f_{1/2}^{-})_{\sharp}s_0$ should be mutually singular, see \ref{append:invSCDT}). 
 \end{proof}
 \begin{remark}\label{rmk:geodesicform}
 	Using the fact that the space $(L^2(s_0)\times \R)^2$ is uniquely geodesic, one can show that if the geodesic in  $\big(S,D_S(\cdot,\cdot)\big)$ exists between two signals $s$ and $\widetilde s$, it should be of the form \eqref{eq:sz}.
 \end{remark}
 \begin{remark}\label{rmk:notconvex}
 	Observing from Proposition \ref{prop:geoconvex}, it is not hard to see that $\widehat S$ is not convex. This is in contrast to the fact the embedding CDT space $S_1^*$  is convex \cite{aldroubi2021partitioning}.
 \end{remark}
 \begin{remark}\label{rmk:unionconvex}
 The set $G$ in Corollary \ref{cor:gengeo} is a group with  the group operation being composition and  defines an equivalence relation in $S$ via $s\sim \widetilde s$ if and only if $\widetilde s = g^{\prime}s\circ g$ for some $g\in G$. It follows that $S$ can be partitioned into disjoint clusters $\bigcup_{i} S_{p_i,G}$ (see \eqref{eq:genmodel}). Hence by Remark \ref{rmk:convexity} $\widehat S$ can be partitioned as a union of convex subsets  $\bigcup_{i}\widehat{ S}_{p_i,G}$.
 \end{remark}
 
 \subsection{Examples of Geodesics}\label{sec:geodesicexamples}
In this section, we give two examples (Figure \ref{fig:TE1-2}) where a unique geodesic exists between the given signals and two examples (Figure \ref{fig:TE5-6}) where a geodesic does not exist. By Remark \ref{rmk:geodesicform}, the geodesic $p_{\alpha}$ for $\alpha\in [0,1]$ is of the form \eqref{eq:sz} when it exists. In the case when a geodesic does not exist, we plot a path in $S$ according to \eqref{eq:sz}.  For a geodesic $p_{\alpha}$ with $p_0=s$ and $p_1 = \widetilde s$, it follows by definition that $\sum_{i=1}^nD_S(p_{\alpha_{i-1}},p_{\alpha_{i}})= D_S(s,\tilde s)$ for  $0=\alpha_0<\alpha_1<\cdots<\alpha_n=1$; while $\sum_{i=1}^nD_S(p_{\alpha_{i-1}},p_{\alpha_{i}})> D_S(s,\tilde s)$ if $\alpha \mapsto p_{\alpha}, \alpha\in [0,1]$ is not a geodesic path.  In particular, in the following plots the signals $p_{\alpha_i}$ where $\alpha_i = \frac{i}{4}, i = 0,\cdots,4$ are presented and distances $D_i := D_S(p_{\alpha_{i-1}}, p_{\alpha_{i}}), i = 1,2,3,4$  and $D:=D_S(s,\tilde s) $ are shown. 

\vspace{-1em}

\begin{figure}[h!]
    \centering
    \includegraphics[width=0.49\textwidth]{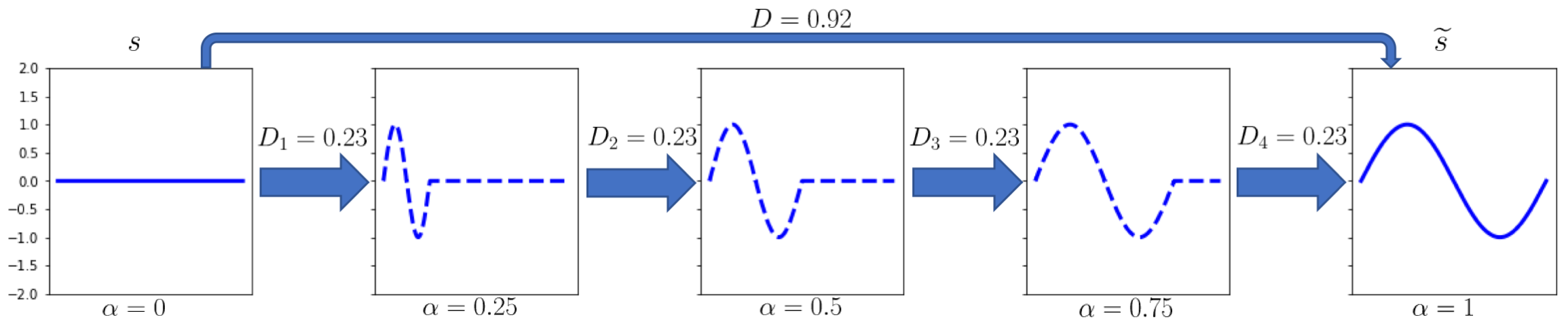}
    \includegraphics[width=0.49\textwidth]{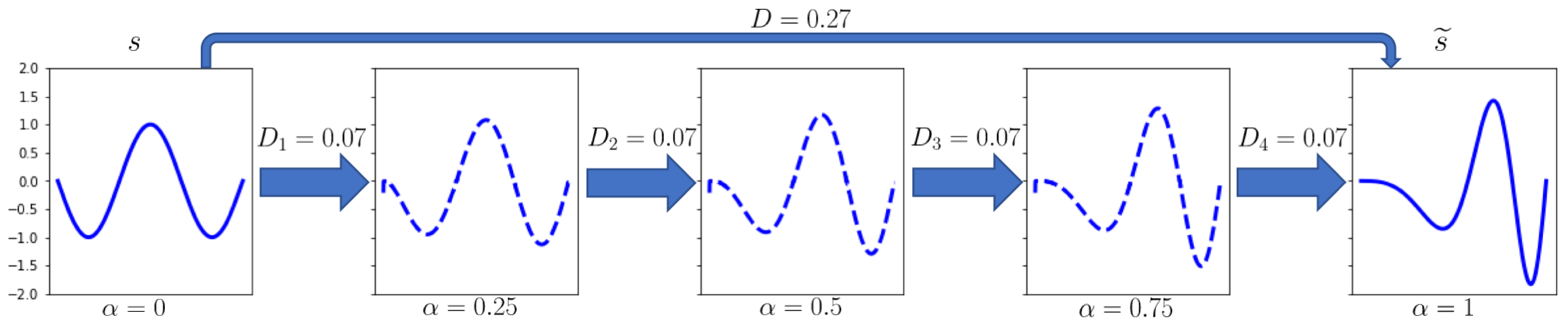}
    \vspace{-1.5em}
    \caption{Examples of geodesics: (top) the geodesic between the zero signal and $s(t)= \sin(2\pi t)\mathbbm{1}_{[0,1]}(t)$ where $p_{\alpha}(t)= \sin(2\pi \frac{1}{\alpha} t)\mathbbm{1}_{[0,\alpha]}(t)$.\footnotemark ~ (bottom) the geodesic between $s(t)= -\sin(3\pi t)\mathbbm{1}_{[0,1]}(t)$ and $\widetilde s = g^{\prime}s\circ g$ where $g(t)= t^2$. In both examples, $D=\sum\limits_{i=1}^4D_i$.}
    \label{fig:TE1-2}
\end{figure}
\footnotetext{By a direct computation (see \ref{sec:zerogeod}), we have that the geodesic is given by $p_{\alpha}(t)= s(\frac{1}{\alpha} t)$ for $\alpha\in (0,1]$ and $p_0:= 0$.}

\vspace{-1.5em}
\begin{figure}[h!]
    \centering
    \includegraphics[width=0.49\textwidth]{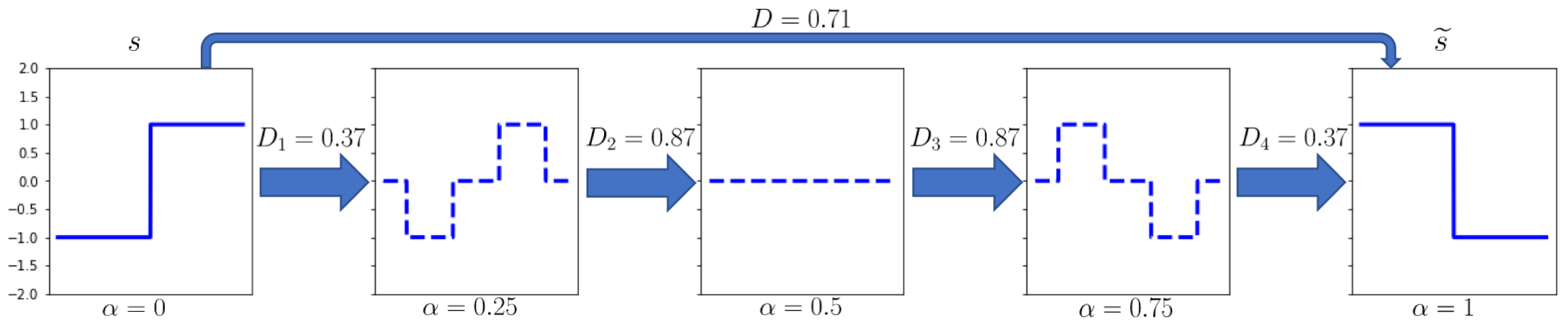}
    \includegraphics[width=0.49\textwidth]{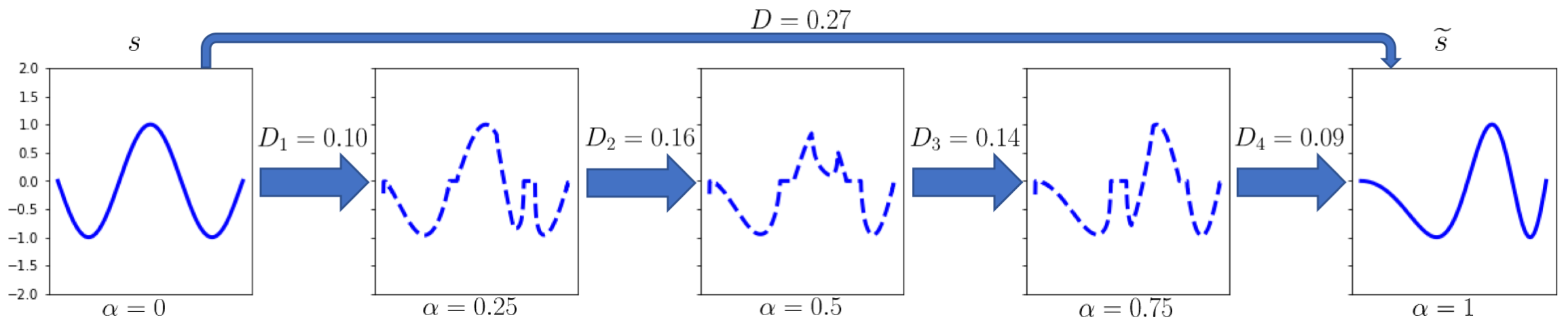}
    \vspace{-1.5em}
    \caption{Examples where geodesics do not exist: (top) signal path $p_{\alpha}$ defined by \eqref{eq:sz}  between $s= -\mathbbm{1}_{[0,.5]}+\mathbbm{1}_{(.5,1]}$ and $\widetilde s = \mathbbm{1}_{[0,.5]}- \mathbbm{1}_{(.5,1]}$; for this path  $D=.71\ll \sum\limits_{i=1}^4D_i = 2.48$. (bottom) signal path $p_{\alpha}$ defined by \eqref{eq:sz} between $s$ and $\widetilde s = s\circ g$ where $g(t)= t^2$ and $s(t)= -\sin(3\pi t)\mathbbm{1}_{[0,1]}(t)$. Note the $\widetilde s$ here differs from that in second example of Figure \ref{fig:TE1-2} (there is no normalizing term $g^{\prime}(t)$ here); for path  $D=.27\ll \sum\limits_{i=1}^4D_i = .49$.}
    \label{fig:TE5-6}
\end{figure}


\section{A Preliminary Application : Interpreting the SCDT Subspace Classifier}\label{sec:interpretability}
In this section, we attempt to utilize the preliminary geodesic properties discussed above to partially interpret the decisions made by the SCDT subspace classifiers proposed in \cite{rubaiyat2022nearest} and \cite{rubaiyat2022nls}, which were shown to achieve very high accuracy in classifying segmented time series events.


\subsection{SCDT subspace classifiers}
The signed cumulative distribution transform (SCDT) combined with a subspace classifier has recently been shown very effective in classifying time series data \cite{rubaiyat2022nearest,rubaiyat2022nls}. In \cite{rubaiyat2022nearest}, the authors employed a nearest subspace search technique in SCDT space to classify time series events that follow a certain generative model:
$ S^{(c)}=\{s_j^{(c)}|s_j^{(c)}=g'_j\varphi^{(c)}\circ g_j, g_j\in G^{(c)}, g'_j>0\}$,
where $s_j^{(c)}$ a signal in class $c$ deformed from $\varphi^{(c)}$ (a template pattern corresponding to  class c), and $G^{(c)}$ denotes a set of increasing deformations of a specific kind (e.g. translation, scaling, etc.).  In the embedding (transform) space, each signal class is hypothesized to be modeled well by the containing subspace $\widehat {V}^{(c)}= \textrm{span}\big(\widehat{S}^{(c)}\big) $, and  the corresponding classifier  searches for the nearest subspace to the test sample to predict its class label.

An extension of the SCDT nearest subspace approach was proposed in \cite{rubaiyat2022nls}, where a more general multi-template generative model was used, assuming that a signal class is generated from a set of templates under some unknown deformations. Formally, the multi-template generative model for signal class $c$ is defined to be the set: $S^{(c)} = \bigcup\limits_{m=1}^{M_c} S_{\varphi_m^{(c)},G_m^{(c)}}$, where $S_{\varphi_m^{(c)},G_m^{(c)}}=\left\{s_{j,m}^{(c)}|s_{j,m}^{(c)}=g'_{j}\varphi_m^{(c)}\circ g_{j}, g'_{j}>0, g_{j}\in G_m^{(c)}\right\}$
and $\left(G_m^{(c)}\right)^{-1}=\left\{\sum_{i=1}^k\alpha_i f_{i,m}^{(c)}, \alpha_i\geq 0\right\}$ lies in a linear space of deformations. Each signal class is hypothesized to be modeled well by a union of subspaces $ \bigcup\limits_{m=1}^{M_c} \widehat{{V}}^{(c)}_m $ in the embedding (transform) space, where $\widehat{V}^{(c)}_m =\textrm{span}\big(\widehat{S} _{\varphi_m^{(c)},G_m^{(c)}}\big)$. The corresponding classifier searches for the nearest local subspace $\widehat{V}^{(c)}_{m*}$ to the test sample to predict the class label. 

\subsection{Interpretation of the SCDT Subspace Classifiers}
As explained in the section above, the SCDT  subspace classifiers search for the nearest subspace $\widehat{V}^{(c)}$ \cite{rubaiyat2022nearest} or local subspace $\widehat{V}^{(c)}_{m*}$ \cite{rubaiyat2022nls} to a given test signal in SCDT (embedding) space. In particular, for the nearest subspace classifier, it can be shown that under certain assumptions,  $d^2(\widehat{s},\widehat{V}^{(c)})<d^2(\widehat{s},\widehat{V}^{(p)})$ for $s\in S^{(c)}$ and $p\neq c$, where 
$d^2(\cdot,\cdot)$ denotes the $L^2$-metric. Similar properties hold for the nearest local subspace classifier. The class labels are then predicted  by searching for the (local) subspace with the smallest distance to the test signal.

In this paper, we provide a preliminary interpretation of the subspace classifiers using a synthetic dataset and a real dataset \textit{StarLightCurves} \cite{rebbapragada2009finding} hosted by the UCR time series classification archive \cite{UCRArchive2018}. In particular, we present a visualization of the paths between a test signal and its projections onto the subspaces associated with different classes.
The projection $\widetilde s$ of a test signal $s$ onto the subspace $\widehat{V}^{(c)}$ associated with class $c$ is defined in the sense that $\widehat {\widetilde s}:= P_{\widehat{V}^{(c)}}\widehat s$, where $P_{\widehat{V}^{(c)}}$ denotes the $L^2$-projection operator to the subspace $\widehat{V}^{(c)}$. 

If the classifier can successfully model a subspace (using training data) that contains a test signal $s$, then the projection $\widetilde s=s$.
In practice, we observe that the path defined by \eqref{eq:sz} between a test signal $s$ and its projection $\widetilde s$ to a subspace associated with the same class as $s$ consists of signals of similar shapes (or looks like a geodesic lying in the same class) and $s\approx \widetilde s $. On the other hand, the path between $s$  and its projection to a subspace associated with a different class seems to exhibit unpredictable behavior.

\begin{figure}[tb]
    \centering
    \includegraphics[width=0.49\textwidth]{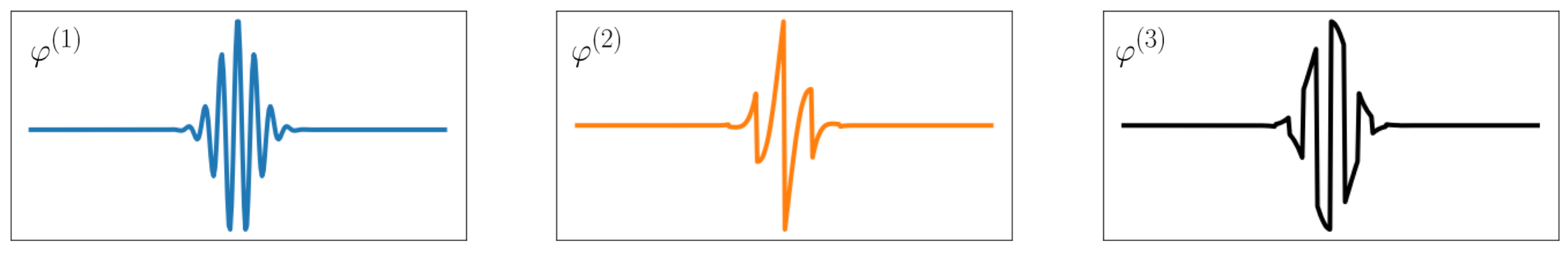}
    \vspace{-1.5em}
    \caption{Three prototype (template) signals used to generate synthetic data in Experiment 1.}
    \label{fig:sig_proto}
    \vspace{-1.5em}
\end{figure}
\textit{Experiment 1:} For this experiment, we created a synthetic dataset of three classes consisting of scaled and translated signals of three template signals: Gabor wave, apodized sawtooth wave, and apodized square wave, respectively (shown in Figure \ref{fig:sig_proto}). The training samples were generated following the generative model given by:
   $ S^{(c)} = \{g'\varphi^{(c)}\circ g|g\in G\},$ where $ G=\{g = \omega t + \tau,\omega>0,\tau\in\mathbb{R}\}$
and $\varphi^{(c)}(t)$ denotes the template corresponding to class $c$ (see Figure \ref{fig:sig_proto}). Here the test sample (not present in training set) follows the generative model for class 1 (Gabor wave).  The signal paths $p_\alpha$ between the test signal and its projections onto three subspaces formed via the subspace classifier \cite{rubaiyat2022nearest} are shown in Figure \ref{fig:toy_classify1}. We observe that the projection onto class 1 resembles the test signal and the corresponding path looks like a geodesic (i.e., the ratio $\sum\limits_{i=1}^4D_i/D\approx 1$),
indicating the subspace $\widehat{V}^{(1)}$ generated by the subspace classifier contains the SCDTs of the test signal $s$ and its cluster $S_{s,G}$ \footnote{Since the same phenomenon is observed for any random test signal following the generative model $S_{\varphi^{(1)},G}$, which equals $S_{s,G}$ for an arbitrary signal $s$ in class 1 since $G$ is the set of affine deformations.}.
However the paths from the test signal to its projections onto the other two classes contain signals that belong to neither classes and seem to exhibit unpredicted patterns.

\begin{figure}[!h]
    \centering
    \includegraphics[width=0.48\textwidth]{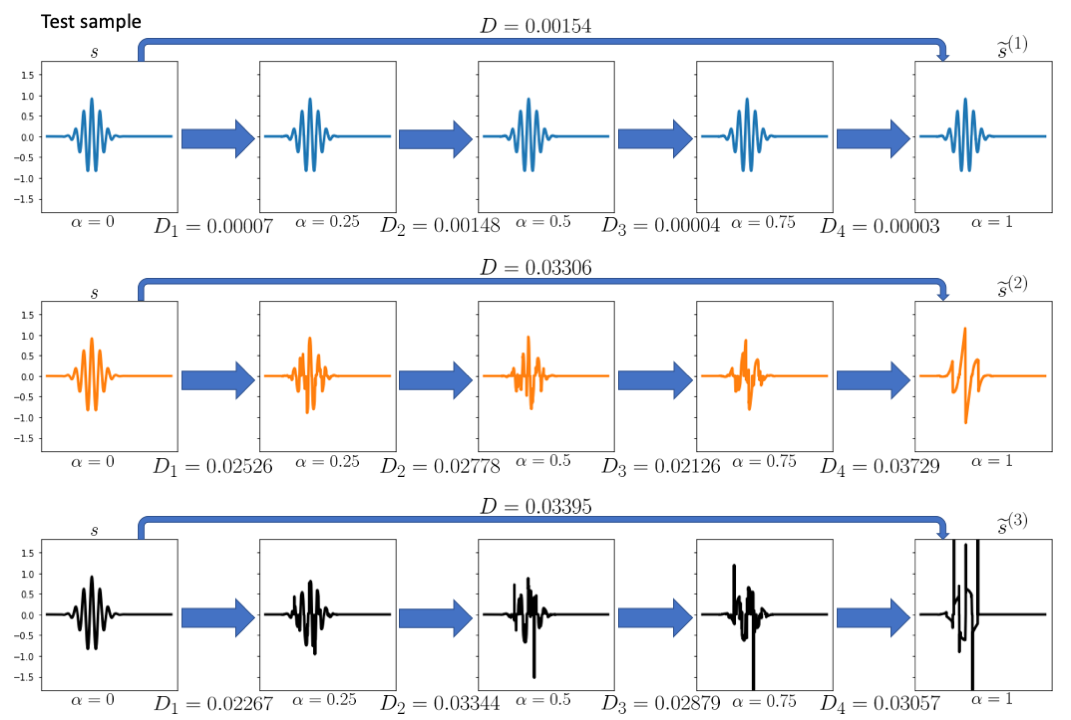}
    \vspace{-1.5em}
    \caption{Signal path $p_{\alpha}$ defined by eq. (\ref{eq:sz}) between a test sample $s$ from class 1 and its projections $\tilde s^{(1)},\tilde s^{(2)},\tilde s^{(3)}$ associated with the subspaces obtained by the subspace classifier \cite{rubaiyat2022nearest} corresponding to the three classes, respectively. Note that the distance $D$ between $s$ and $\widetilde s^{(1)}$ (i.e., $D_S(s,\widetilde s^{(1)})$) is the smallest, and hence the classifier predicts the correct label.}
    \vspace{-1.0em}
    \label{fig:toy_classify1}
\end{figure}

\begin{figure}[tb]
    \centering
    \includegraphics[width=0.49\textwidth]{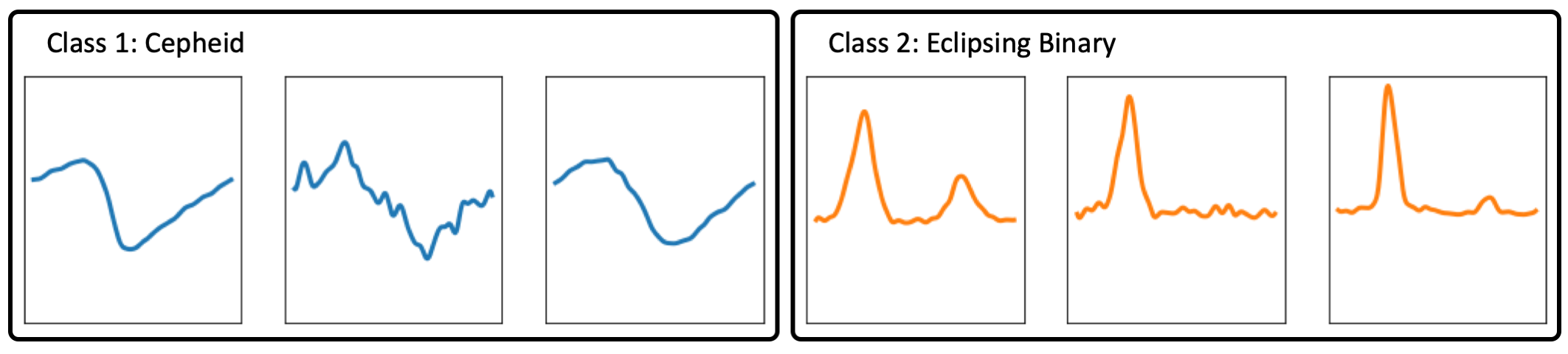}
    \vspace{-1.5em}
    \caption{Samples from two classes in the \textit{StarLightCurves} dataset}
    \label{fig:sig_data}
    \vspace{-1.0em}
\end{figure}

\textit{Experiment 2:} We conducted a similar experiment with a real dataset \textit{StarLghtCurves} \cite{rebbapragada2009finding}, where  two classes: `Cepheid' (class 1) and `Eclipsing Binary' (class 2) were used. Sample signals from these two classes are shown in Figure \ref{fig:sig_data}. Signal paths $p_{\alpha}$ are presented in Figure \ref{fig:real_classify1} from a test signal $s_j$ in each class ($j=1,2$) to their projections $\widetilde s^{(i)}_j$ corresponding to subspaces $\widehat{V}^{(i)}$ ($i= 1,2$) produced by the local subspace classifier proposed in \cite{rubaiyat2022nls}. We observe that the path from $s_j$ to  its projection $\widetilde s_j^{(j)}$ associated with the same class resembles a geodesic within class $j$ (see the first and the fourth row in Figure \ref{fig:real_classify1}) while the path from $s_j$ to its projection $\widetilde s_j^{(i)}$ associated with a different class $i\neq j$ contains very noisy signals which seem to belong to neither classes (see the second and third row of Figure \ref{fig:real_classify1}).

\begin{figure}[!h]
    \centering
    \includegraphics[width=0.49\textwidth]{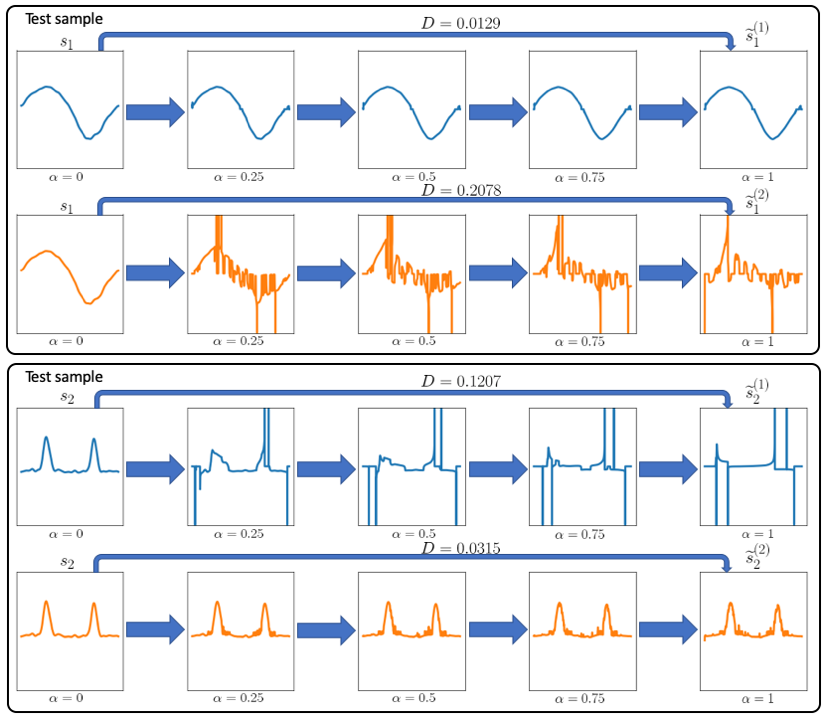}
    \vspace{-1.5em}
    \caption{Signal path $p_{\alpha}$ defined by eq. (\ref{eq:sz}) between a test sample $s_j$  and its projections $\tilde s^{(i)}_j$ ($i,j=1,2$) in the local subspaces obtained by the subspace classifier \cite{rubaiyat2022nls} corresponding to two classes, respectively. Note that in the top panel the distance $D$ between $s_1$ and $\widetilde s^{(1)}_1$ (i.e., $D_S(s_1,\widetilde s^{(1)}_1)$) is smaller, and in the bottom panel the distance $D$ between $s_2$ and $\widetilde s^{(2)}_2$ is smaller. Hence the local classifier predicts the correct labels in both cases.}
    
    \label{fig:real_classify1}
\end{figure}

\begin{remark}  
    Given the stark differences between the path from a signal to its projection associated with a subspace of its own class and that to its projection associated with a subspace of a different class, one may consider using the length of the path (see e.g., $\sum\limits_{i}D_i$ in Figure \ref{fig:toy_classify1}) as the classification metric, in contrast to the $D_S$ distance between the signal and its projection used in the aforementioned subspace classifiers, which we leave to future work.
\end{remark}

\section{Conclusion}\label{sec:discussion}
In this preliminary study, we look into the geometry of time series with respect to a generalized Wasserstein metric and its embedding in a linear
space. In particular, a geodesic may not exist between two arbitrary time series in $S$ but exists between signals within the same generative cluster (see \eqref{eq:genmodel}).  We utilize this knowledge to show whether the (local) subspaces formed by the training samples via the SCDT subspace classifiers do a good job in modeling  the signal classes conforming to certain generative modeling assumptions. We illustrate it by visualizing the paths between random test signals and their projections onto the corresponding subspaces associated with different classes and observe that there is a  path resembling a geodesic from a test signal to its projection in the subspace associated with the same class of the test signal. This shows some preliminary evidence that the classifiers produce  a ``good" subspace that models the corresponding generative clusters to which the test signals belong. In the future, we are interested  in  quantitative (numerical) characterization of geodesics for real/noisy data in the hope of building more robust models and classifiers for time series events.

\nocite{langley00}

\bibliography{ref}
\bibliographystyle{icml2022}

\newpage
\appendix
\onecolumn

\section{Appendix}
\subsection{An additional  proposition}
\begin{proposition}\label{prop:geoconvex}
	Let $U\subseteq S$ such that $\widehat U$ is a convex subset of $\widehat S$. Then there exists a unique constant speed geodesic with respect to $D_S(\cdot,\cdot)$ between any two signals in $U$.
\end{proposition}

\begin{proof}
	Let $u,w\in U$. Since $\widehat U$ is convex, $(1-\alpha) \widehat u+\alpha \widehat w \in \widehat U \subseteq \widehat S$ for any $\alpha\in [0,1]$. Since the SCDT is a bijection between $S$ and $\widehat S$ (see \ref{rmk:bijection}), there exists a unique $u_{\alpha}\in U$ such that $\widehat u_{\alpha}= (1-\alpha) \widehat u+\alpha \widehat w$ for any $\alpha \in [0,1]$. Note that $u_0 = u$ and $u_1=w$.  By the embedding property \eqref{eq:scdtembedding} of the SCDT, we have that
\begin{align*}
	D_S(u_{\alpha_1},u_{\alpha_2})&= \nm{\widehat u_{\alpha_1}- \widehat u_{\alpha_2}} _{(L^2(s_0)\times \R)^2}\\
	&= |\alpha_1-\alpha_2|\nm{\widehat u - \widehat w}_{(L^2(s_0)\times \R)^2}\\
	& = |\alpha_1-\alpha_2|D_S(u,w),
\end{align*}
which shows that $\alpha \mapsto u_{\alpha}$ defines a constant speed geodesic in $U$.
\end{proof}

\subsection{Push-forward of density functions}
The push-forward operator is more generally defined for measures. Here we present its analog for the associated density functions. Given a $L^1$- normalized density function $s:\Omega_s\rightarrow \R$ ($s
\geq 0$) and a (measurable) map $T:\R\rightarrow \R$, the push forward density function $T_{\sharp}s\in L^1(\R)$ is defined via \begin{equation}
   \int_{B} \big(T_{\sharp}s\big)(y)dy = \int_{T^{-1}(B)\cap \Omega_s} s(t)dt,
\end{equation}
for any measurable set $B\subseteq \R$.

\subsection{A characterization of the CDT}\label{append:charCDT}
The following fact is a corollary of Theorem 2.9 in \cite{santambrogio2015optimal}.
\begin{proposition}
	Given $s_1,s_2\in S_1$, there exists a unique optimal transport map $T^*$ (i.e., a map such that $T^*_{\sharp} s_1 = s_2$ which minimizes the Wasserstein-2 cost) between them given by the non-decreasing map
\begin{equation}
	T^* = F_{s_2}^{\dagger}\circ F_{s_1}.
\end{equation}
On the other hand, if $T^*:\Omega_{s_1}\rightarrow \Omega_{s_2}$ is a non-decreasing map such that $T^*_{\sharp} s_1 = s_2$, it follows that  $T^*$ is the optimal transport map between $s_1$ and $s_2$. 
\end{proposition}

\subsection{The composition property of the CDT and SCDT}\label{appendix:composition}
Let $g:\R\rightarrow \R$ be a strictly increasing and differentiable function. Here we present the formulas for the transform (CDT or SCDT) of $s_g:= g^{\prime}s\circ g$ for some $s\in S$. 

For $s\in S_1$, the composition property for the CDT is given by \cite{Park:18}
\begin{equation}
	s_g^* = g^{-1}\circ s^*.
\end{equation}

For $s\in S$ such that $\nm{s^{+}}\neq 0$ and $\nm{s^{-}}\neq 0$,  the composition property for the SCDT is given by \cite{aldroubi2022signed}
\begin{equation}
	\widehat s_g = \Big(g^{-1}\circ (s^{+})^*, \nm{s^{+}}, g^{-1}\circ (s^{-})^*,\nm{s^-}\Big).
\end{equation}

Similarly, for $s$ such that $\widehat s =\Big ((s^{+})^*, \nm{s^{+}},0,0\Big)$ or $\widehat s = (0,0,(s^{-})^*, \nm{s^{-}})$ or $s=0$, $\widehat s_g$ is given by $\Big(g^{-1}\circ (s^{+})^*, \nm{s^{+}},0,0\Big)$, $\Big(0,0,g^{-1}\circ(s^{-})^*, \nm{s^{-}}\Big)$ and $(0,0,0,0)$, respectively.

\subsection{On invertibility of the SCDT}\label{append:invSCDT}

		Let $s_0\in S_1$ be a $L^1$-normalized positive signal supported on $\Omega_{0}$ and let  $s\in S$ . Then by the definition of the CDT and Jordan decomposition, it is not hard to see that $(s^+)^*_{\sharp}s_0 \perp (s^-)^*_{\sharp}s_0$ since $(s^+)^*_{\sharp}s_0 =s^+ $ and  $(s^-)^*_{\sharp}s_0 =s^- $. Here ``$\perp$" means that the measure with density $(s^+)^*_{\sharp}s_0$ and the measure with density $(s^-)^*_{\sharp}s_0$ are mutually singular. In particular, there 
		exists $\Omega_{+}$ and	 $\Omega_{-}$ 		such that $\Omega_s = \Omega_{+}\cup \Omega_{-}$ with $\int_{\Omega_{-}}s^+(x)dx =0 $ and $\int_{\Omega_{+}}s^-(x)dx =0$.
	

The following fact is a special case of Theorem 2.7 in \cite{aldroubi2022signed}.
\begin{proposition}\label{prop:invSCDT}
Given a fixed a $L^1$-normalized positive signal reference signal $s_0\in S_1$ supported on $\Omega_{0}$. Then for any tuple $(f,a,g,b)$ satisfying 
\begin{equation}
	f_{\sharp} s_0 \perp g_{\sharp} s_0,
\end{equation}
where $f,g\in L^2(s_0) $ are non-decreasing $s_0$-a.e. and $a,b>0$, there is a unique $s\in S$ such that $\widehat s = (f,a,g,b)$ given by the following inverse formula
\begin{equation}
	s = af_{\sharp}s_0- bg_{\sharp}s_0.
\end{equation}
Moreover, the unique inverses for  tuples of the forms $(f,a,0,0)$, $(0,0,g,b)$ and $(0,0,0,0)$ are $af_{\sharp} s_0$, $bg_{\sharp} s_0$ and the zero signal respectively. 

\begin{remark}\label{rmk:bijection}
	From the above proposition, we see that the SCDT defines a bijection from $S$ to $\widehat S$ via $s\mapsto \widehat s$ and $\widehat S = \{\widehat s: s\in S\} = \{(f,a,g,b): f,g\in L^2(s_0) ~\textrm{non-decreasing}~s_0~\textrm{-a.e.}, f_{\sharp} s_0 \perp g_{\sharp} s_0, a,b>0 \}\cup \{(f,a,0,0): f\in L^2(s_0)~\textrm{non-decreasing}~s_0~\textrm{-a.e.}, a>0\}\cup\{(0,0,g,b): g\in L^2(s_0)~\textrm{non-decreasing}~s_0~\textrm{-a.e.}, b>0\}\cup \{(0,0,0,0)\}$.
\end{remark}
\end{proposition}

\subsection{Length of a path and geodesics in metric spaces}
Let $(V,d)$ be a metric space and $p: [0,1]\rightarrow V$ be a curve (path) in $V$. Then the length of $p$ is defined by:
\begin{equation}
    \textrm{Len}(p):= \sup\Big\{\sum\limits_{i=0}^{n-1} d(p(\alpha_i),p(\alpha_{i+1})): n\geq 1, 0=\alpha_0<\alpha_1<\cdots<\alpha_{n}\Big\}.
\end{equation}

A curve $p: [0,1]\rightarrow V$ between $v$ and $w$ in $V$ is called a geodesic if it is a length-minimizing curve:
\begin{equation}
    \textrm{Len}(p) = \min\{\textrm{Len}(\gamma): \gamma:[0,1]\rightarrow V, \gamma(0)=v, \gamma(1)=w\}.
\end{equation}
Moreover, it is called a
constant-speed geodesic if \begin{equation}
    d(p(\alpha),p(\beta))=|\alpha-\beta|d(p(0),p(1)), \quad \forall \alpha,\beta \in [0,1].
\end{equation}

A metric space  $(V,d)$ is said to be a geodesic space if for any $v,w \in V$, there exists a curve $p: [0,1]\rightarrow V$ with $p(0)=v$ and $p(1)=w$ such that
$d(v,w)= \textrm{Len}(p)$. Indeed, this curve $p$ is a geodesic between $v$ and $w$.

\subsection{Uniqueness of geodesics in normed spaces with strictly convex norms}
We present the following property of strictly convex spaces (see Proposition 1.6 in \cite{bridson2013metric}):
\begin{lemma}
	Every normed vector space $V$ is a geodesic space. It is uniquely geodesic if and only if the unit ball in $V$ is strictly convex (in the sense that if $u$ and $w$ are distinct vectors of norm $1$, then $\nm{(1-\alpha)u+zw}<1$ for all $\alpha\in (0,1)$. The geodesic between $v,w\in V$ is given by $\alpha\mapsto (1-\alpha)v+\alpha w$ for $\alpha\in [0,1]$.

\end{lemma}

\begin{corollary}\label{cor:unigeo}
Let $s_0\in S_1$. The space $(L^2(s_0)\times \R)^2=\big\{(f,a,g,b): f,g \in L^2(s_0), a,b\in \R\big\}$ of tuples endowed with the $2$-norm $\nm{(f,a,g,b)}_2 := \sqrt{\nm{f}^2_{L^2(s_0)}+a^2+\nm{g}^2_{L^2(s_0)}+b^2}$ is uniquely geodesic, with the unique geodesic between $(f_0,a_0,g_0,b_0)$ and   $(f_1, a_1,g_1,b_1)$ given by $f_{\alpha} = (1-\alpha)(f_0,a_0,g_0,b_0) + \alpha (f_1, a_1,g_1,b_1)$ with $\alpha\in [0,1]$.
\end{corollary}

\subsection{Geodesics between the zero signal and an arbitrary signal}\label{sec:zerogeod}

Let $s\in S$. There exists a unique constant-speed geodesic in $\big(S, D_S(\cdot,\cdot)\big)$ between $0$ and $s$  given by $p_{\alpha}(t)= s(\frac{1}{\alpha} t)$ for $\alpha\in (0,1]$ and $p_0:= 0$.  By direct computation, for  $\alpha_1,\alpha_2\in (0,1]$
\begin{align}
    D_S^2(p_{\alpha_1},p_{\alpha_2})&= \nm{(p_{\alpha_1}^{+})^*-p_{\alpha_2}^{+})^*}^2 + \nm{(p_{\alpha_1}^{-})^*-p_{\alpha_2}^{-})^*}^2 + \Big(\nm{(p_{\alpha_1}^{+}}-\nm{(p_{\alpha_2}^{+}}\Big)^2 + \Big(\nm{(p_{\alpha_1}^{-}}-\nm{(p_{\alpha_2}^{+}}\Big)^2\nonumber\\
    &= \nm{\alpha_1(s^{+})^*-\alpha_2(s^{+})^*}^2_{L^2(s_0)}+\nm{\alpha_1(s^{-})^*-\alpha_2(s^{-})^*}^2_{L^2(s_0)} + \Big(\alpha_1\nm{(s^{+}}-\alpha_2\nm{s^{+}}\Big)^2 + \Big(\alpha_1\nm{s^{-}}-\alpha_2\nm{s^{-}}\Big)^2\nonumber\\
    &  = (\alpha_1-\alpha_2)^2 \Big(\nm{(s^{+})^*}^2_{L^2(s_0)} + \nm{(s^{-})^*}^2_{L^2(s_0)} + \nm{s^{+}}^2 + \nm{s^{-}}^2\Big) \nonumber\\
    & = (\alpha_1-\alpha_2)^2\norm\Big{\Big(0,0,0,0\Big)-\Big((s^{+})^*,  \nm{s^{+}}, (s^{-})^*, \nm{s^{-}}\Big)}_{(L^2(s_0)\times \R)^2}^2\nonumber\\ 
    &= (\alpha_1-\alpha_2)^2D_S^2(0,s)\nonumber,
\end{align}
where the second equality follows from the composition property of the SCDT. Similarly one  can verify that $D_S(0,p_{\alpha}) = \alpha  D_S(0,s)$ for $\alpha\in [0,1]$.


\end{document}